\documentclass[twoside,11pt]{article}

\usepackage{blindtext}
\usepackage{amsmath}
\usepackage{array} 
\usepackage{booktabs} 
\usepackage{color}
\usepackage{graphicx}
\usepackage{subcaption}  

\usepackage{jmlr2e}
\newtheorem{assumption}[theorem]{Assumption}

\newcommand{\dataset}{{\cal D}}
\newcommand{\W}{{\cal W}}

\providecommand{\comjh[1]}{\com{JH}{#1}}
\providecommand{\com[2]}{\begin{tt}[#1: #2]\end{tt}}
\providecommand{\prt}[1]{\left( #1 \right)}
\providecommand{\abs[1]}{\left|#1\right|}
\providecommand{\norm[1]}{\abs{\abs{#1}}}%

\providecommand{\R}{\mathbb{R}}
\providecommand{\lessim}{\lesssim}

\newcommand{\ao}[1]{{#1}}
\newcommand{\modif}[1]{#1}
\newcommand{\aoo}[1]{#1}

\usepackage{lastpage}

\ShortHeadings{Convex SGD: Generalization Without Early Stopping}{J. Hendrickx and A. Olshevsky}
\firstpageno{1}

\begin{document}

\title{Convex SGD: Generalization Without Early Stopping}

\author{\name Julien M. Hendrickx \email julien.hendrickx@uclouvain.be\\
       \addr ICTEAM Institute,        
       UCLouvain,\\
       Louvain-la-Neuve, B-1348, Belgium
       \AND
       \name Alex Olshevsky \email alexols@bu.edu \\
       \addr Department of Electrical and Computer Engineering \\
       Division of Systems Engineering \\ 
       Boston University, Boston, USA}

\editor{}

\maketitle

\begin{abstract}
We consider the generalization error associated with stochastic gradient descent on a smooth convex function over a compact set. We show the first bound on the generalization error that vanishes when the number of iterations $T$ and the dataset size $n$ go to zero at \modif{arbitrary} rates; our bound scales as $\tilde{O}(1/\sqrt{T} + 1/\sqrt{n})$ with step-size  $\alpha_t = 1/\sqrt{t}$.  In particular, strong convexity is not needed for stochastic gradient descent to generalize well. 
\end{abstract}

\begin{keywords}
  Stochastic Gradient Descent, Generalization 
\end{keywords}

\section{Introduction}

Stochastic Gradient Descent (SGD) is ubiqutious in machine learning and it is  of considerable interest to understand  its  generalization abilities,  i.e.,  to quantify the performance of models trained with SGD to handle unseen data.  There is a recent line of literature, initiated in \cite{hardt2016train} that suggests SGD has a limited ability to overfit, regardless of the number of parameters in the model. This observation appears to match the commonly-observed phenomenon of practitioners using variations on gradient descent to successfully train vastly over-parametrized models.

In this paper, we will consider the problem of using SGD to fit a smooth convex loss function to the data. In this case, SGD can be written in the form,
\[ w_{t+1} = w_t - \alpha_t \left( \nabla_{w} \ell(w_t; S) + \epsilon_t\right), \] where for any dataset $S$, $\epsilon_t$ is noise which is zero-mean  conditioned on on $w_0, \ldots, w_t$ and  $\ell(w_t;S)$ is the empirical risk suffered by the iterate $w_t$ on the data set. Formally, we assume there is a set of data points $S=\{z_1, \ldots, z_n\}$ generated i.i.d. by sampling from some distribution ${\mathcal D}$ and  there is some loss function $f(w;z)$ such that, 
\[ \ell(w;S) = \frac{1}{|S|} \sum_{z_i \in S} f(w; z_i). \]
The goal is to bound the generalization error, which is the expected loss of  $w_t$ (or a running average of $w_t$) on a new sample generated from the same distribution ${\mathcal D}$: 
\[ G(w_t) = E_{z \sim {\cal D}} [f(w_t;z)] - \ell (w_t,S). \] 

Under the assumption that the function $f(w,z)$ is convex and has gradients with Lipschitz constant $L$ that are further always bounded by $B$ in the Euclidean norm, it was shown in \cite{hardt2016train} that 
\begin{equation} \label{eq:rechtbound} G(\widehat{w}_t) \leq \frac{2B^2}{n} \sum_{t=1}^T \alpha_t,
\end{equation} where $\widehat{w}_t$ is a certain running average of the iterates. In particular, an implication of \modif{this result} is that if we run gradient descent for a fixed number of iterations, the generalization error is finite and scales as $\sim 1/n$. 

The main question we want to address in this paper is: {\em can we derive a bound in this setting that does not go to infinity with the number of iterations $T$?} Indeed, common step-sizes for SGD often scale as $ \sim 1/\sqrt{t}$ or $ \sim 1/t$, and in both of these cases the sum on the right-hand side of Eq. (\ref{eq:rechtbound}) is infinite. In general, attaining a finite generalization bound from Eq. (\ref{eq:rechtbound}) \modif{for a given $n$ and arbitrarily large $T$} means choosing a step-size schedule $\alpha_t$ whose sum is finite, but this typically rules out the possibility of attaining the same training error as the empirical minimum $\min_{w} l(w;S)$:  analyses of SGD and other optimization methods require the sum of the step-sizes to be infinite. 

There have been many follow-up works, discussed in the literature review below, that have used variations on Eq. (\ref{eq:rechtbound}) to get good generalization bounds; however, this required ``early stopping,'' i.e., choosing $T$ depending on problem parameters. 
We are motivated by the observation that practitioners often do not employ early stopping and so, to the extent that  generalization bounds for gradient descent might one day  explain the surprisingly good performance of SGD, it is useful to understand whether Eq. (\ref{eq:rechtbound}) might be replaced by something uniformly bounded over $T$.

\bigskip
{\bf The convexity assumption.} 
The kind of bounds we are looking for are available in the strongly convex case \modif{(on a bounded domain)}: it was  shown in \cite{hardt2016train} that
\begin{equation} \label{eq:rechtstrongly} G(w_t) \leq \frac{2 B^2}{\mu n}, 
\end{equation} independently of $t$ where $\mu$ is the strong convexity parameter. {\aoo Not only is this bound finite independently of how long one runs SGD for, but also it goes to zero as the dataset size $n$ goes to infinity}. However,  for the class of smooth convex functions without strong convexity, such a bound is unavailable. One reason to be particularly interested in the convex but not strongly convex case is that it comes up naturally in many scenarios throughout machine learning. 

Indeed, let us consider the simplest \modif{case} of using a linear model \modif{for classification}. In this scenario, we want to train a classifier on a collection of data points of the form $(x_i,y_i)$ where $y_i$ is a binary $\pm 1$ label. \modif{We take the simple} 
linear model $w^T x \approx y$; a convex loss is desirable, and we may attempt to use the strongly convex quadratic loss $(1/2) (w^T x - y)^2$. 

A potential problem with doing so is that we will eventually use $w^T x$ to classify a data point $x$ to $\pm 1$, and it makes no difference whether e.g., $w^T x =2$ or $w^T x = 3$, as all of these will be classified to $+1$. This may sometimes give rise to scenarios where loss (either test or train) is falling but the corresponding classification accuracy is getting worse. Indeed, a minimization algorithm with the above quadratic loss might find that updating $w$ to move various $x_i$ with $w^T x_i > 1$ closer to $w^T x_i \approx 1$ is better in terms of reducing the loss than updating $w$ in a way that improves classification accuracy. 

We may therefore attempt to strengthen the connection between loss and classification accuracy by introducing a one-sided loss, 
\[ f(w; (x_i,1)) = \begin{cases} 0 & w^T x \geq 1 \\ 
(1/2) (w^T x - 1)^2 & \mbox{ else} \end{cases} \] and likewise for $f(w; (x_i, -1))$. It  is also possible to consider a version where the breakpoint between two regimes occurs when $w^T x = 0$. Either way we now have a loss that is now convex and smooth but not  strongly convex. 


\medskip

\noindent {\bf Problem Statement.} Given the discussion above, we can now state our goal more formally. {\aoo In optimization theory, to get convergence of SGD to an optimal solution one usually takes a step-size $\alpha_k$ satisfying 
\begin{equation} \label{eq:optstepsize} \sum_{k=1}^{+\infty} \alpha_k = +\infty, ~~~ \frac{\sum_{k=1}^T \alpha_k^2}{\sum_{k=1}^T \alpha_k} \rightarrow 0. 
\end{equation} Using any such step-size,} we would like to derive a bound of the form 
\[ G ( \widehat{w}_T) \leq q(d, n, T), \] under the assumption that the loss is convex, with bounded and Lipchitz gradient, and where the bound $q(d,n,T)$  will not diverge to infinity as $T \rightarrow +\infty$ (while $n$ and $d$ are constant). Further, we want that as $T \rightarrow \infty$, the function $q(d,n,T)$ approaches something that scales inversely with a function of the dataset size $n$. 
We can rephrase this by saying that we would like to have that $q(d,n,T) \rightarrow 0$ as long as $d$ is fixed and $n$ and $T$ go to infinity at arbitrary rates. 

\subsection{Literature Review}  The study of the  effectiveness of large-scale stochastic optimization methods in generalizing was initiated in \cite{bottou2007tradeoffs}, \modif{who} recognized three key elements that impact the generalization characteristics: errors in optimization, errors in estimation, and approximation errors. Generalization bounds on SGD were derived in \cite{hardt2016train} based on the idea of algorithmic stability \cite{bousquet2002stability}: it was argued that SGD generalizes well because it is relatively insensitive to the removal/addition of a single point to the data set $S$. The last work  inspired a fairly large follow-up literature, which is too numerous to survey in its entirety; we will focus \modif{only on the results most relevant to our purpose.}

In the case where the functions are strongly convex, it is possible to obtain a generalization bound that is finite regardless of how long one runs SGD for; this was already in \cite{hardt2016train}. Such a bound was later also  obtained for functions satisfying the PL-type conditions in \cite{charles2018stability, lei2021sharper, liuoptimization, li2021improved}. A finite generalization bound in the strongly convex case was also present in \cite{london2017pac} which derived SGD generalization bounds using a PAC-Bayes framework. A finite bound was also derived in the non-convex setting in \cite{farghly2021time}, but in the setting where an independent batch can be drawn at any iteration. A generalization to mirror descent and relative strong convexity was considered in \cite{attia2022uniform}. Perturbations of strongly convex functions were considered in \cite{park2022generalization}, with a leading-in-$T$ term that scaled as $T/n$. {\aoo Related to this literature are the bounds obtained on the sample complexity of SGD with Langevin dynamics in \cite{raginsky2017non} subject to a dissipativity condition on the underlying function.}

There are a number of possible assumptions one can consider which lead to different bounds. For example, in the non-smooth but still convex case, \cite{bassily2020stability} derived generalization bounds where the leading term in $T$ scales as $T/n$. In the non-convex but smooth case \cite{neu2021information} building on work by \cite{russo2016controlling, xu2017information} derived a bound in terms of pathwise statistics of SGD, and whose leading term is linear in $T$. The earlier work  \cite{kuzborskij2018data} also considered the non-convex case, with a leading term in $T$ that grows as a power of $T$ determined by certain problem parameters. {\aoo Closely related is the paper \cite{lugosi2022generalization}, where an information-theoretic approach to generalization is derived for single-pass gradient descent, i.e., where each data point is used exactly once; in our terminology, this may be viewed as a case of ``early stopping.''}  

Most closely related to our work is \cite{lei2020fine}, which consider an almost identical setting. The generalization bound showed there for the non-strongly convex case (Theorem 4) provides a bound without the assumption of bounded gradients. Although it grows like $T/n^2$ with \modif{the number $T$ of iterations}, it can be made on the order of $1/\sqrt{n}$ with early stopping by choosing $T \sim n$ (the final scaling is $\sim 1/\sqrt{n}$ due to the presence of other terms in the bound). A similar idea was explored in \cite{rosasco2015learning} in the linear case which observed that tuning the number of passes through the data is akin to tuning a regularization parameter. Bounds on the expected risk were derived in \cite{rosasco2015learning} which scaled as $T^2$ but became negligible in the size of the data set $n$ after a judicious choice of early stopping. 
A high probability version of the bound of \cite{hardt2016train} in the non-strongly-convex case was recently derived in \cite{yuan2022boosting}, scaling as $\sqrt{T}/n$.

{\aoo Another similar work to ours is \cite{pensia2018generalization}, which derives an information-theoretic bound for SGD. Specialized to our assumption where we assume the true gradient is corrupted by noise of constant second moment, the bound of \cite{pensia2018generalization} grows as $\sqrt{T/n}$ An exposition of this result, along with a survey of the entire area of generalization bounds, can be found in the monograph \cite{hellstrom2023generalization}.}

Likewise, in \cite{lin2016generalization} a step-size proportional to $\sim 1/\sqrt{t}$ was explored, with the resulting generalization error having a leading term with $T$ that scales as $\sqrt{T}/n$. An improved generalization bound was also derived in \cite{feldman2018generalization}, also with leading term that is $\sqrt{T}/n$.  Under the same smooth/convex/non-strongly-convex assumptions, \cite{deng2021toward} derived a bound that scaled as $\sim (1/n) \sum_{t=1}^T \alpha_t$; the difference from \cite{hardt2016train} was that the constant in front of the bound was not the same.

Most of these works rely on arguments based on algorithmic stability. Unfortunately, in \cite{zhang2022stability} a lower bound on the most common notion of algorithmic stability was proven that grows proportionally to the sum of step-sizes $(L/(2n)) \sum_{t=1}^T \alpha_t$. This suggests that different tool(s) are necessary to obtain bounds that do not blow up over an infinite time horizon.  


Relatedly, \cite{schliserman2022stability} considers a non-strongly-convex setting under the assumption of self-boundedness, i.e., \[ ||\nabla f(w;z)|| \leq c f(w;z)^{1-\delta},\] and obtain a bound on generalization error whose leading term in $T$ scales as 
$T^{\delta}/n$, with possibly some additional multiplicative terms depending on the interplay between $T$  and the growth of the norm in approximating the minimizer.  

A popular approach is to instead consider the regularized loss
\[ l_{\lambda}(w;S)  = l(w;S) + \lambda ||w||^2, \] and use the fact that the generalization error associated with strongly convex function is finite regardless of how many iterations of SGD one does. For example, this is what is done in Chapter 13 of the textbook \cite{shalev2014understanding}. Relatedly, in the recent work \cite{lei2018stochastic}, it is assumed that the quantity 
\[ D_{\lambda} =  \inf_{w} E_{z \sim {\cal D}} [f(w_t;z)]  + \lambda ||w||^2 - \inf_{w}   E_{z \sim {\cal D}} [f(w_t;z)]\]
satisfies the inequality
\[ D_{\lambda} \leq c_{\alpha} \lambda^{\alpha}.\] Under these conditions, the leading term of the generalization error was shown to grow at the rate of $n^{-\alpha/(1+\alpha)} \log^{2/3} T$.

\renewcommand{\arraystretch}{1.5}
\begin{table}[h!]
\centering
\begin{tabular}{|m{4cm}||m{6cm}|m{4cm}|}
\hline
\textbf{Previous Work} & \textbf{Asymptotic Growth With $T$} & \textbf{Relies on Stability} \\
\hline
\cite{hardt2016train}&  $(1/n) \sum_{t=1}^T \alpha_t$ & Yes \\
\hline
\cite{feldman2018generalization} & $\sqrt{T}/n$ & Yes \\
\hline
\cite{pensia2018generalization} & $\sqrt{T/n}$ & No \\ 
\hline 
\cite{lei2020fine}& $T/n^2$ & Yes \\
\hline
\cite{deng2021toward} & $(1/n) \sum_{t=1}^T \alpha_t$ & Yes \\
\hline
\cite{yuan2022boosting} & $\sqrt{T}/n$ & Yes \\
\hline
{\color{blue} This paper}   & {\color{blue} $\sqrt{d/n}$ }  & {\color{blue} No }  \\
\hline
\end{tabular} \label{table:comparison}
\caption{Previous papers that considered generalization error associated with smooth, convex, Lipschitz optimization. The middle column shows the asymptotic growth of the generalization error as $T \rightarrow \infty$ with $n$ being the size of the data set and $d$ being the dimension \modif{of the parameter space}. 
Results on strongly convex, linear, or functions satisfying  PL or dissipativity conditions (or bounds that assume early stopping) are not included in this table.}
\label{tab:my_label}
\end{table}

\subsection{Our  Contribution}

We will further make the assumption that we are working over a compact convex set, so that we will actually consider Projected Stochastic Gradient Descent (PSGD) rather than  SGD. {\em Our main result is to derive a bound on generalization error which is uniformly upper bounded independently of $T$ for any step-size satisfying Eq. (\ref{eq:optstepsize})} 
A summary of how our paper compares to previous results is given in the table {\aoo (see also table legend for criteria for inclusion in table)}.

The purpose of our result to make the point that early stopping is not needed for convex SGD to generalize. On a technical level, almost all previous works in the table relied on stability arguments, which in this context means an analysis of how the performance of SGD changes as we change a {\em single} point in the data set. The sole exception seems to be \cite{pensia2018generalization} (building on earlier arguments of \cite{xu2017information, russo2016controlling}) which used information theoretic arguments. The main technical difference between this work and most of the previous papers is that we do not rely on stability defined in this way {\aoo (nor do we apply an information-theoretic approach)}. Instead, the core of our proof relies on an analysis of perturbed gradient descent which can be applied to the setting when {\em all} points of the underlying data set are changed in a random way. 


It is notable that our result has a generalization error that depends on the dimension $d$, while the other results in the table do not. 
However, it turns out that some kind of dimension scaling is inevitable under our assumptions, as Proposition \ref{prop:lower} in Section \ref{sec:inevitability} demonstrates. {\aoo The lower bound is quite similar to arguments made in \cite{shalev2009stochastic} and \cite{amir2021sgd}, with a few minor modifications made to handle the setting we work in (e.g., we assume Lipschitz smoothness which is not assumed in the previous work). We also mention a recent improvement to \cite{amir2021sgd} in \cite{schliserman2024dimension}.}

{\aoo It is an open question, however, whether the $\sqrt{d/n}$ scaling with dimension in our result is optimal. Note that our bound implies we need $O\left(d/\epsilon^2 \right)$ samples in dimension $d$ to have an $\epsilon$ generalization error for the empirical risk minimizer. This is not tight: a recent paper \cite{carmon2023sample} obtains a $O(d/\epsilon + 1/\epsilon^2)$ for the same sample complexity, which, while being linear in $d$, has $d$ multiplying an asymptotically negligible term as $\epsilon \rightarrow 0$. However, it does not automatically follow that because an improvement is possible for the empirical risk minimizer, an improvement is possible for the entire SGD trajectory.  
}

\bigskip 

%
%
%

\section{Our Main Result}

In this section we will state and prove our main result, the generalization bound for smooth SGD over convex, compact sets. Our proof will be complete with exception to a certain concentration result, which we will postpone to the next section. 

\subsection{Problem Settings and Statement}

We begin by (re)describing the setting we will work in. We will consider $n$ i.i.d. data points $z_1,z_2,\ldots, z_n$ generated from a distribution $\dataset$ and a loss function $f(w;z)$,  resulting in the empirical loss defined as 
\[ \ell(w;S) = \frac{1}{n} \sum_{z_i \in S}  f(w;z_i), \] where, recall, $S$ is our notation for $S=\{z_1, z_2, \ldots, z_n\}$; thus $\ell(w;S)$ is the empirical loss associated with the parameter $w$ and the dataset $S$.  
While no assumptions will be made on $\dataset$, we will make the following assumption on the loss function.

\begin{assumption}\label{assumption:function} The function $f(.;z)$ is defined on the compact convex set $\W \subset \mathbb{R}^d$ which is a subset of the ball of the Euclidean ball of radius $R$ around the origin: $\W \subset B[0,R]$. Moreover, for every $z$, the function $f(\cdot;z)$ is convex and $L$-smooth: 
    \begin{align*}
        || \nabla_w  f(w_1;z)  - \nabla_w  f(w_2;z)  || &\leq L ||w_1 - w_2||,\end{align*} 
        {\aoo Finally, the gradient has bounded variance at a minimizer of the true loss:}
        \[ 
        \aoo{{\rm var} \norm{\nabla_w f(w^*,z)} }  \aoo{\leq (\sigma^*)^2},
    \] 
    { \aoo where ${\rm var}(X)$ is the variance of a random variable $X$ and  $w^* \in \arg \min_{w \in {\cal W}} E_{z \in {\cal D}} f(w;z).$}
\end{assumption}

Note that $f(w;z)$ is assumed to be smooth in the first argument for every $z$. Naturally, in the remainder of the paper $\nabla$ will always mean $\nabla_{w}$, since we will only differentiate the loss function with respect to $w$. \modif{Besides, we do not assume that the compact set $\W$ necessarily contains a stationary point, and indeed the minimum could lie on its boundary and have a nonzero gradient. Hence there is no direct relation between $B$ and $LR$.}


\medskip

We will consider PSGD for the problem of minimizing the empirical loss $\ell(w;S)$ over $w \in \W$. We can write this method as 
\begin{equation}\label{eq:PSGD}
w_{t+1} = P_{\W} \left( w_t - \alpha_t [ \nabla_w \ell(w_t,S) + \epsilon_t] \right),
\end{equation}
where $P_{\W}$ denotes the projection onto the set $\W$, and $\epsilon_t$ is a random variable satisfying 
$$
E [\epsilon_t | w_0, \ldots, w_t] = 0 \hspace{1.5cm} 
E [||\epsilon_t||^2 | w_0, \ldots, w_t] \leq \sigma^2. 
$$
 \modif{without any other  assumptions, in particular the $\epsilon_t$ are not necessarily independent of each other.} This is a fairly general form of the update, and we remark that it includes  as a special case the method 
\begin{equation}\label{eq:Proj_block_sgd} w_{t+1} = P_{\W} \left( 
w_t - \alpha_t \ao{ \frac{1}{|{\cal B}_t|}} \sum_{z_i \in {\cal B}_t} \nabla_{w} f(w_t;z_i)
\right),
\end{equation} 
where elements of the batch ${\cal B}_t$ are drawn uniformly from $S$. We also remark that gradient descent is a special case of this problem formulation corresponding to $\sigma^2=0$, and nothing in the sequel requires $\sigma^2$ to be strictly positive.

When we need to emphasize the dependence of the iterates produced by Eq. (\ref{eq:PSGD}) on the set $S$, we will write $w_t(S)$. However, sometimes this dependence will be clear from context, in which case we will simply write $w_t$ as above. 

\medskip

A standard analysis of PSGD (e.g., \cite{nemirovski2009robust}) gives the following convergence bound:
\begin{eqnarray*}
    E \left[ \ell(\widehat{w}_t(S_n);S_n) \right] - \min_{w \in \W} \ell(w;S_n) & \leq & \frac{ ||w_0 - w^*(S_n)||^2 }{2\sum_{k=0}^t \alpha_k} + 
\sigma^2
\frac{\sum_{k=0}^t \alpha_k^2}{\sum_{k=0}^t \alpha_k}  \\ 
    & := & \modif{\bar \epsilon_{\rm opt}}(t,n),
\end{eqnarray*}
 where the second line indicates  that we will denote the entire bound by $\modif{\bar \epsilon_{\rm opt}}(t,n)$; and $w^*(S_n)$ is our notation for a minimizer, i.e., 
\[ w^*(S_n) \in \arg \min_{w \in \W} \ell(w; S_n),\] picked arbitrarily in case of multiple minimizers; and $\widehat{w}_t(S_n)$ is a running average of the iterates, 
\begin{equation}\label{eq:def:running_avg}
\widehat{w}_t(S_n) = \frac{\sum_{k=0}^t \alpha_k w_{k+1}(S_n)}{\sum_{k=0}^t \alpha_k}.
\end{equation} 

\medskip

The generalization error, which we will denote by $\epsilon_{\rm gen}(t,n)$, is then the expected increase in loss of a {\em new} datapoint $z$ when using $\widehat{w}_t(S_n)$. Formally:  
\begin{equation}\label{eq:def_gen_error}
    \epsilon_{\rm gen}(t,n) = E_{z \sim {\cal D}, S_n \sim {\cal D} } \left[ f(\widehat{w}_t(S_n);z)- \ell(\widehat{w}_t(S_n); S_n) \right],
\end{equation}
where the generation of the new data point $z \sim {\cal D}$ is independent from the generation of the original data set $S_n \sim {\cal D}$, \modif{and we remind the reader that the empirical loss $\ell$ is the average loss $f$ over the points in the training set}.

Our main result is then the following theorem. 

\medskip

\begin{theorem} \label{thm:smooth}  Under Assumption \ref{assumption:function} \modif{and provided the step sizes are bounded as $\alpha_t\leq \frac{1}{L}$,} we have

\begin{equation}\label{eq:first_bound_gen_error}
\epsilon_{\rm gen}(t,n) \leq \modif{\bar  \epsilon_{\rm opt}}(t,n) +   O \left( \frac{\aoo{\sigma^* R}+LR^2\sqrt{d}}{\sqrt{n}} \right).   
\end{equation}
 \end{theorem} 

 \bigskip

To parse the theorem statement, recall that $n$ is the size of the dataset, $B$ and $L$ are a bound on the gradient and its Lipschitz constant respectively, while, in the constraint $w \in \W$, the set $\W$ was assumed to be a subet of the ball of radius $R$ around the origin in  $\mathbb{R}^d$. 

Note that this theorem bounds $\epsilon_{\rm gen}(t,n)$ as a sum of two terms. The first term goes to zero with iteration $t$ {\ao for any choice of step-size which {\aoo satisfies Eq. (\ref{eq:optstepsize}).  }
Such step-sizes include $1/t^{\alpha}$ for $\alpha \in {\aoo [}1/2,1]$ and are common in the optimization literature. 

The second term is fixed independently of the number of steps one does. \modif{In particular, it applies if we let the algorithm converge to a minimizer of the empirical loss, and bounds thus the generalization error at the minimizer independently of how it was computed.}
We note that it depends on the data set as $O(1/\sqrt{n})$, so that it can be made arbitrarily small by taking a large sample. 

It is tempting to take the popular step-size $\alpha_t=1/\sqrt{t}$, in which case our theorem bounds the generalization error {\aoo after $T$ steps} as $\tilde{O}(1/\sqrt{T} + 1/\sqrt{n})$. In particular, we can send $n$ and $T$ to infinity at arbitrary rates and the generalization error will go to zero. This should be contrasted with the bounds obtained in previous work in Table 1, which do not have this property.
 


\subsection{Proof of our main result}

We now provide a proof of our main result, modulo a lemma on concentration whose proof we will defer to the next section. This will require a series of intermediate results and lemmas. Our first lemma is a technical result on projection we will need.

\begin{lemma} \label{lemma:proj_expectation} Let $\W$ be a closed convex set and let $v \in \W$. Suppose that $N$ is a random vector which satisfies $E[N] = 0$. We then have 
\[ \left|  E \left[ N^T P_{\W} \left( v - \alpha N \right) \right] \right|  \leq \alpha E \left[ || N ||^2 \right] \]\end{lemma}

\begin{proof} By non-expansiveness of the projection, we have 
$$
||  P_{\W} \left( v - \alpha N \right) - P_{\W}(v) || \leq |\alpha| ||N||.
$$
Hence:
\begin{align*}
\left| N^T \left(P_{\W} \left( v - \alpha N \right) - P_{\W}(v)\right)  \right| & \leq \max_{y:||y||\leq |\alpha| ||N||} N^Ty \\ & = \alpha ||N||^2
\end{align*}
Now taking an expectation we obtain
\begin{align*}
 \left|  E \left[  N^T P_{\W} \left( v - \alpha N \right) \right] \right|  &= \left| E \left[ N^T P_{\W} \left( v \right) \right]  +  E \left[ N^T \left(P_{\W} \left( v - \alpha N \right) - P_{\W}(v)\right)\right] \right|  \\
 & = \left| E \left[ N^T \left(P_{\W} \left( v - \alpha N \right) - P_{\W}(v)\right)\right] \right| \\ 
 & \leq E \left[ \left| N^T \left(P_{\W} \left( v - \alpha N \right) - P_{\W}(v)\right) \right| \right) \\
  &\leq \alpha E ||N||^2. 
\end{align*}
\end{proof}

Our next lemma gives an analysis of projected gradient descent over a convex set with error, and traces how the errors affect the final performance. It is a modification of an analysis performed in  \cite{devolder2014first} for the unprojected case.

\begin{lemma} \label{lemma:inexact_sgd_orig}  
Consider the update
$$ w_{t+1} = P_{\W} \left( w_t - \alpha_t [ \nabla_w \ell(w_t;S) + p_t + \epsilon_t] \right),$$ 
where $\W \subset B(0,R)$, the weights $\alpha_t$ satisfy  $\alpha_t\leq 1/L$, \modif{$p_t$ are arbitrary perturbations and $\epsilon_t$ random ones.} We will assume that,  \begin{eqnarray*} 
E [\epsilon_t | w_1, \ldots, w_t, p_1, \ldots, p_t ] & = & 0 \\ 
E \left[ || \epsilon_t||^2 | w_1, \ldots, w_t \right] & \leq & \sigma^2  \\ 
||p_t|| & \leq & \bar p_t. 
\end{eqnarray*}
Let $\ell_S^*$ the minimum of $\ell(.;S)$ over $\W$, with $w^*$ arbitrarily selected as a point that achieves that minimum. Then we have 
$$
E \left[ \ell ( \hat{w}_t) - \ell_S^*  \right] \leq \frac{ ||w_0 - w^*||^2 }{2\sum_{t=0}^T \alpha_t} + 
\sigma^2
\frac{\sum_{t=0}^T \alpha_t^2}{\sum_{t=0}^T \alpha_t}  +  \frac{\sum_{t=0}^T \alpha_t \delta_t 2R}{\sum_{t=0}^T \alpha_t} ,
$$
\modif{for the running average $\hat{w}_t$ defined as in \eqref{eq:def:running_avg}.}

\end{lemma} 

\begin{proof} For simplicity of notation, we use the notation,  $$\tilde g_t := \nabla_w \ell (w_t;S) + p_t +\epsilon_t$$ so that the PSGD  iteration can be compactly written as,
$$
w_{t+1} = P_{\W} \left( w_t - \alpha_t \tilde g_t \right).
$$ 

We first argue that 
\begin{align}
||w_{t+1} - w^*||^2 - ||w_t - w^*||^2 & = || P_{\W} (w_t - \alpha_t \tilde{g}_t) - w^*||^2  - ||w_t - w^*||^2 \nonumber  \\
&\leq ||w_{t} - \alpha_t \tilde{g}_t -w^*||^2 -  ||w_t - w^*||^2 \nonumber,
\end{align} using the nonexpansiveness of projection. 
Next, we use the identity, 
\[ 
||y-z||^2 = ||x-z||^2  - 2 \modif{\langle x-y,y-z \rangle} - ||y-x||^2,
\]
with $z=w^*, x = w_t, y = w_{t+1}$ to obtain that      

\begin{equation}  ||w_{t} - \alpha_t \tilde{g}_t -w^*||^2 -  ||w_t - w^*||^2  = - 2 \alpha_t \tilde{g}_t^T (w_{t+1}-w^* )  - ||w_{t+1} - w_t||^2 \nonumber
\end{equation} 
Putting this together, we obtain 
\begin{align}
||w_{t+1} - w^*||^2 - ||w_t - w^*||^2 
& \leq  - 2 \alpha_t \tilde{g}_t^T (w_{t+1}-w^* )  - ||w_{t+1} - w_t||^2 \nonumber \\ 
&=  - 2 \alpha_t \left[\nabla_w \ell(w_t,S)^T (w_{t+1}-w^* )  +\frac{1}{2\alpha_t} ||w_{t+1} - w_t||^2\right] \label{eq:basis_dec_cvx} \\
& - 2\alpha_t p_t^T (w_{t+1}-w^* ) - 2\alpha_t \epsilon_t^T (w_{t+1}-w^* )\label{eq:basis_dec_other}.
\end{align}

We begin by analyzing the term between parenthesis in \eqref{eq:basis_dec_cvx}. Standard convexity results together with $1/(2\alpha_t) \geq L/2$ imply that \begin{small} 
\begin{align}
    \nabla_w \ell(w_t,S)^T (w_{t+1}-w^* )  +\frac{1}{2\alpha_t} ||w_{t+1} - w_t||^2&=   \nabla_w \ell(w_t,S)^T (w_{t+1}-w_t ) +\frac{1}{2\alpha_t} ||w_{t+1} - w_t||^2 \nonumber \\ & ~~~~~~ - \nabla_w \ell(w_t,S)^T (w^*-w_t ) \nonumber  \\
    &\geq (\ell(w_{t+1}) - \ell(w_t)) - (\ell(w^*) - \ell(w_t)) \nonumber  \\ & = \ell(w_{t+1}) - \ell(w^*) \nonumber 
\end{align}
\end{small} 
We next consider bounding (\ref{eq:basis_dec_other}). We use Cauchy-Schwarz to obtain that 
$$|2\alpha_t p_t^T (w_{t+1}-w^* )|\leq 2\alpha_t \bar{p}_t (2R), $$ where, recall, $\W$ is contained in $B(0,R)$ and $\bar{p}_t$ is a bound on $||p_t||.$ 

Moving to the remaining term in (\ref{eq:basis_dec_other}), we see that
\begin{align*}
- 2\alpha_t  \epsilon_t^T (w_{t+1}-w^* ) & = -2\alpha_t  \epsilon_t^T (w_{t+1} - w_t) - 2\alpha_t   \epsilon_t^T (w_{t} - w^*)\\
&= -2\alpha_t  \epsilon_t^T\left( P_{\W} \left( w_t - \alpha_t \nabla_w \ell(w_t,S) - \alpha_t p_t  - \alpha_t  \epsilon_t \right)  - w_t \right) -  2\alpha_t   \epsilon_t^T (w_{t} - w^*)
\end{align*}
We have
\[ E \left[ - 2\alpha_t \epsilon_t^T (w_{t+1}-w^* )\right] = -2 \alpha_t E [  \epsilon_t^T  P_{\W} (v-\alpha_t  \epsilon_t) ],  \]
where $$ 
v = w_t - \alpha_t \nabla_w \ell(w_t,S) - \alpha_t p_t.$$
We use Lemma \ref{lemma:proj_expectation} \modif{and the assumption 
$E [\epsilon_t | w_1, \ldots, w_t, p_1, \ldots, p_t ]  =  0 $}
to obtain, 
$$
E[- 2\alpha_t  \epsilon_t^T (w_{t+1}-w^* ) | w_t] \leq 2\alpha^2_t\sigma^2.
$$
Re-introducing all these developments into \eqref{eq:basis_dec_cvx} and \eqref{eq:basis_dec_other} yields
$$
E[||w_{t+1} - w^*||^2 | w_t] - ||w_t - w^*||^2 \leq -2\alpha_t (E[\ell(w_{t+1})|w_t] - \ell(w^*)) +  2\alpha^2_t\sigma^2 + 2\alpha_t \bar{p}_t (2R).
$$
Taking an additional expectation and re-arranging, we obtain
\begin{equation}
   \alpha_t \prt{ E(\ell(w_{t+1})) - \ell(w^*) }\leq \frac{E[||w_{t} - w^*||^2 ] - E[ ||w_{t+1} - w^*||^2 ]}{2}  + \alpha^2_t \sigma^2 + \alpha_t p_t (2R).
\end{equation}
Summing over all $t$ and dividing by the sum of the weights yields 
$$
\frac{\sum_t  \alpha_t E(\ell(w_{t+1}))}{\sum_t  \alpha_t } - \ell(w^*) \leq \frac{||w_0-w^*||^2}{2\sum_t  \alpha_t} + \sigma^2 \frac{\sum_t  \alpha_t^2}{\sum_t  \alpha_t} + 2R \frac{\sum_t  \alpha_t p_t}{\sum_t  \alpha_t},
$$
from which the lemma follows by convexity. 
\end{proof}

Our next proposition is a slight restatement of the definition of generalization.

\begin{proposition} \label{prop:G_t-transform} 
\[ \epsilon_{\rm gen}(t,n) =  E \left[ \ell(\widehat{w}_t(S_n), S_n') - \ell(\widehat{w}_t(S_n'), S_n') \right], \] where $S_n'$ is another data set of the same size $n$ sampled i.i.d. from ${\cal D}$.
\end{proposition} 

\begin{proof} Indeed, 
\begin{eqnarray*}
E  \left[ \ell(\widehat{w}_t(S_n), S_n') \right]  & = & E \left[  \frac{1}{n} \sum_{z_i' \in S'} f(\widehat{w}_t(S_n); z_i')  \right] \\ 
& = & E \left[ f(\widehat{w}_t(\modif{S_n}); z_1') \right]
\end{eqnarray*} where the last equality follows the $z_i'$ here are all sampled i.i.d. from the distribution ${\cal D}$ so the expectation of each term in the sum on the first line is identical. Observing that $S_n$ and $S_n'$ are identically distributed and using the definition of $\epsilon_{\rm gen}(t,n)$ completes the proof. 
\end{proof}

\bigskip 

Next, given two arbitrary sets $S, S'$ \modif{of data points}, we can define the measure
\begin{equation}\label{eq:def_DeltaSS'}
\Delta(S,S')  = \sup_{w \in \W}|| \nabla_w  \ell (S, w) - \nabla_w \ell (S', w)||. 
\end{equation}
Intuitively, $\Delta(S,S')$ measures the difference between the two sets of $S$ and $S'$ in terms of the averages of the gradients of $f(w;z)$ over points in $z \in S$ vs. $z \in S'$.

We can bound this difference in the following lemma. We will use the notation 
\[ a \lessim b,\] which is just a compact substitute for $O(\cdot)$ notation: it means there exists a positive absolute constant $K$ such that $a \leq Kb$.

\begin{lemma} \label{lemma:concentration} Under Assumption \ref{assumption:function} and supposing that $S,S'$ are sets of size $n$ drawn i.i.d. from some distribution ${\mathcal D}$, 
\[E  \left[ \Delta(S,S') \right] \lesssim  \frac{\aoo{\sigma^*}+LR\sqrt{d}}{\sqrt{n}}.\] 
\end{lemma}

We postpone the proof of this lemma until the next section. With in place, we can now finish the proof of our main result, Theorem \ref{thm:smooth}. 

\medskip 

\begin{proof}[Proof of Theorem \ref{thm:smooth}] Fix $S_n,S_n'$ and let us define
\[ w^*(S') \in \arg \min_{w \in \W} \ell (w,S_n'), \] where, if the set of minimizers on the right-hand side has cardinality greater than one,  we pick one arbitrarily and designate it to be $w^*(S_n')$. 

We view the PSGD process on $\ell(w,S_n)$ as an inexact gradient descent process on $\ell(w,S_n')$. 
Indeed, we can take the PSGD dynamics 
\[ w_{t+1} = w_t - \alpha_t \left(\nabla \modif{\ell} (w_t, S_n) + \epsilon_t \right),
\] and rewrite them as 
\[ w_{t+1} = w_t - \alpha_t \left(\nabla \modif{\ell}(w_t, S_n') + \left[ \nabla \modif{\ell}(w_t, S_n) - \nabla{\ell}(w_t, S_n') \right] + \epsilon_t \right),
\] and then view the term in brackets as an arbitrary perturbation \modif{$p_t$ (and moreover the random variable $\epsilon_t$ introduced by the dynamics is independent of $p_t$).}
Thus we can apply Lemma \ref{lemma:inexact_sgd_orig} with 
\begin{align*} \bar{p}_t & = \sup_{w \in \W} || \nabla \modif{\ell} (w_t, S_n) - \nabla \modif{\ell}(w_t, S_n') ||  \\ 
& = \Delta(S_n,S_n'),
\end{align*} where the second equation is just the definition of $\Delta(S_n,S_n')$. The guarantee of Lemma \ref{lemma:inexact_sgd_orig} gives: 
\begin{equation} \label{eq:sub-bound-1}  E [ \ell (\hat{w}_t(S_n), S_n') - \ell (w^*(S_n'), S_n')  ~|~ S_n,S_n'] \leq \frac{ ||w_0 - w^*||^2 }{2\sum_t \alpha_t} + 
\sigma^2
\frac{\sum_t \alpha_t^2}{\sum_t \alpha_t}  +  2R \Delta(S_n,S_n'). \end{equation}

The above is true for any choice of $S_n, S_n'$. We now take expectations with respect to $S_n,S_n'$, and appeal to  Proposition \ref{prop:G_t-transform}:
\begin{align*}  \epsilon_{\rm gen}(t,n) & \leq \frac{ ||w_0 - w^*||^2 }{2\sum_t \alpha_t} + 
\sigma^2
\frac{\sum_t \alpha_t^2}{\sum_t \alpha_t} +  2R \Delta(S_n,S_n') \\ 
& = \modif{\bar  \epsilon_{\rm opt}}(t,n) + 2 R E \Delta(S_n, S_n')
\end{align*} Finally, we use  Lemma \ref{lemma:concentration} to bound the last term. This concludes the proof. 
\end{proof}

\section{Concentration of gradient differences\label{sec:concentration}}

The purpose of this section is to prove Lemma \ref{lemma:concentration}, \modif{which will complete the proof Theorem \ref{thm:smooth}}. Specifically, our goal is to bound the quantity 
$\Delta(S_n,S_n')$ defined in \eqref{eq:def_DeltaSS'} as
\[ \sup_{w \in \W} || \nabla_{w} \ell (S_n,w) - \nabla_{w} \ell (S_n',w)||. \]
Now for each $w$ this is quite easy but what is needed is to bound the sup; the required concentration result for this is almost, but not quite, \modif{sub-Gaussian}. Dudley's inequality \cite{dudley1967sizes} allows bounding the supremum over all $w$, but we need to modify it to handle random variables that are not quite \modif{sub-Gaussian}, which will be the object of the first subsection. The second will be dedicated to the application of this bound in our context.

\subsection{Modified Dudley's Inequality}

We consider a set $\W$ endowed with a metric $d(w_1, w_2)$, and assume the set $W$ is contained within a ball of radius $\widehat{R}$ around the origin in that metric. 
We further consider a random variables $Z_w$ for each $w\in \W$, assumed vary smoothly with $w$ in the following sense:
\begin{assumption}[Increment concentration] \label{ass:increment_concentration} For some $K \geq 1$, 
\[ P( | Z_{w_1} - Z_{w_2}| \geq u) \leq 2 K e^{-u^2/d^2(w_1,w_2)}\]
\end{assumption} 
This assumption says that the differences between the random variables concentrate in almost \modif{sub-Gaussian} manner. The difference is the factor of $K$ multiplying the bound on the right-hand side. As far as we are aware, this assumption cannot be restated in terms of having some other random variable be \modif{sub-Gaussian}.

We will bound $E \sup_{w\in  \W} Z_w$ in terms of the expected value of $Z_w$ at the origin and of some integral that depends on the so-called covering number of $\W$: For every $\epsilon >0$ we define a finite subset $G_\epsilon\subset \W$ with the property that every $w\in \W$ is at a distance at most $\epsilon$ from a point in $G_\epsilon$ in terms of the metric $d(.,.)$. We further assume that the number of point in $G_\epsilon$ is the smallest possible over all sets satisfying this property, and denote it by $N_\epsilon$. Clearly, $N_\epsilon$ is non-increasing with $\epsilon$ as for every $\epsilon'>\epsilon$, the set $G_\epsilon$ is at a distance no larger than $\epsilon'$ from every point in $\W$. Besides, for every $\epsilon>\widehat R$ we can build a $G_\epsilon$ consisting of one single point, so that $N_\epsilon = 1$. We will refer to the $G_\epsilon$ as grids. We are now ready to state the our modification of Dudley's bound, reminding that $a \lesssim b$ means there is an absolute constant $C$ such that $a \leq C b$.

\begin{theorem}  \label{theorem:dudley} Under Assumption \ref{ass:increment_concentration},
\begin{equation} E \sup_{w \in \W} Z_w \lesssim E Z_{w_0} + \sum_{j > i} 2^{-j} \sqrt{\log(K N_{2^{-j}})}. \label{eq:dudley_discrete_form}
\end{equation} 
where $i$ is the largest integer such that $2^{-i} \geq \widehat R$. Likewise, 
\begin{equation} E \sup_{w \in \W} Z_w \lesssim E Z_{w_0} + \int_0^{\widehat{R}} \sqrt{\log(K N_{\epsilon})} ~ d \epsilon, \label{eq:dudley_integral_form}
\end{equation}
\end{theorem} 

We now proceed to prove Theorem \ref{theorem:dudley}. Our argument is a modification of the elegant proof of Dudley's inequality from \cite{talagrand1996majorizing}.

We define $\pi_j(w)$ as the projection of $w\in \W$ onto the grid $G_{2^{-j}}$, i.e. the point in $G_{2^{-j}}$ closest to $w$, breaking ties in an arbitrary manner. Observe that by definition and by the triangular inequality, we have
\begin{equation}\label{eq:distance_proj}
d(\pi_j(w),\pi_{j-1}(w)) \leq 2^{-j} + 2^{-(j-1)} \leq 2^{-j+2}
\end{equation}
Our next lemma bounds how much $Z_w$ can vary when we move from the projection $\pi_{j-1}(w)$ to the projection $\pi_{j}(w)$ on a finer grid, uniformly on $w$ and $j$.

 \begin{lemma} \label{lemma:Zdiff} 
  Define 
 \[ M_{2^{-j}} = N_{2^{-j}} N_{2^{-(j-1)}}, \] and set 
 \begin{eqnarray*} a_j &= 4 \cdot 2^{-j} \sqrt{\log(K2^{j-i} M_{2^{-j}})} \end{eqnarray*} 
 Then for all $u \geq 1$,
 \[ P \left( \sup_{w\in \W, j > i} | Z_{\pi_j(w)} - Z_{\pi_{j-1}(w)} | \geq u a_j\right) \lesssim 2^{-u^2},
 \] where we recall that $i$ is the largest integer such that $2^{-i} \geq R.$ 
 \end{lemma} 

\begin{proof} 
For a given $w\in \W$ and a given $j$, it follows from \eqref{eq:distance_proj} and Assumption \ref{ass:increment_concentration} that for any choice of $a_j$ (not necessarily the one made in the lemma statement)
$$
P \left(  | Z_{\pi_j(w)} - Z_{\pi_{j-1}(w)} | \geq u a_j \right) \leq 2K {\rm exp} \left( -\frac{u^2 a_j^2}{(2^{-j+2})^2}\right). 
$$
We will now take the supremum over $w$. Observe that there are at most $M_{2^{-j}} = N_{2^{-j}} N_{2^{-(j-1)}}$ couples $(\pi_{j-1}(w),\pi_{j}(w))$. Therefore the union bounds implies
$$
P \left( \sup_{w\in \W} | Z_{\pi_j(w)} - Z_{\pi_{j-1}(w)} | \geq u a_j \right) \leq 2K M_{2^{-j}} {\rm exp} \left( -\frac{u^2 a_j^2}{(2^{-j+2})^2}\right). 
$$
We then using again the union bound to take the supremum over $j$,
  \[
P \left( \sup_{j > i} | Z_{\pi_j(w)} - Z_{\pi_{j-1}(w)} | \geq u a_j\right) \leq 2 K\sum_{j>i}    M_{2^{-j}} {\rm exp} \left( -\frac{u^2 a_j^2}{(2^{-j+2})^2} \right)
 \]
For the specific choice of $a_j$ made in that theorem, we have 
\[
P \left( \sup_{j > i} | Z_{\pi_j(w)} - Z_{\pi_{j-1}(w)} | \geq u a_j\right) \leq 
2 K \sum_{j>i}  M_{2^{-j}} (K 2^{j-i} M_{2^{-j}})^{-u^2}.
\] Since $u \geq 1$ and $ K \geq 1$, we can conclude that 
\begin{eqnarray*} 
P \left( \sup_{j > i} | Z_{\pi_j(w)} - Z_{\pi_{j-1}(w)} | \geq u a_j\right) & \leq & 
2 \sum_{j>i} (2^{j-i})^{-u^2} \\ 
& \lesssim & 2^{-u^2} + 2^{-2u^2} +  2^{-3u^2} + \cdots 
 \lesssim  2^{-u^2} 
\end{eqnarray*} 
\end{proof} 

With the above inequality in place, we can prove the concentration result of Theorem \ref{theorem:dudley}. 

\begin{proof}[Proof of Theorem \ref{theorem:dudley}]
Naturally, 
\begin{eqnarray} 
E \sup_{w \in \W} Z_w & = &  E \sup_{w \in \W} (Z_w - Z_{w_0}) + E Z_{w_0} \nonumber \\ 
& \leq & E \sup_{w \in \W} |Z_w - Z_{w_0}| + E Z_{w_0} \label{eq:separate_w0}
\end{eqnarray} We  now focus on bounding the first term of the last equation. We  have that with probability one, 
\[ \lim_{j \rightarrow +\infty} Z_{\pi_j(w)} - Z_w = 0, \] which is formally proved by combining the Borel-Cantelli lemma and equivalence of norms with Assumption \ref{ass:increment_concentration}. We can rewrite this as  
\[ Z_{w} - Z_{w_0} = (Z_{\pi_{i+1}(w)} - Z_{w_0}) + (Z_{\pi_{i+2}(w)} - Z_{\pi_{i+1}(w)}) + \cdots,  \] where, recall, $i$ was chosen to be the largest integer such that $2^{-i} \geq \widehat{R}$. 

\modif{Therefore,} as a consequence of Lemma \ref{lemma:Zdiff}, with the choice of $a_j$ in that lemma, \modif{we have that for every $u\geq 1$}, 
\[ P \left( \sup_{w \in \W} |Z_w -Z_{w_0}| \geq u \sum_{j>i} a_j \right) \lesssim 2^{-u^2}. \] Integrating this:
\begin{eqnarray*} E \left[ \sup_{w \in \W} |Z_{w} - Z_{w_0}| \right] & = & \int_0^{+\infty} P \left( \sup_{w \in \W} |Z_{w} - Z_{w_0}|  \geq x \right) ~ dx \\ 
& = & \modif{\int_{0}^{\sum_{j>i} a_j} P \left( \sup_{w \in \W} |Z_{w} - Z_{w_0}|  \geq x \right) ~ dx  } \\
& +& \sum_{j>i} a_j \int_1^{+\infty} P \left( \sup_{w \in \W} |Z_{w} - Z_{w_0}|  \geq u \sum_{j>i} a_j \right) ~ du \\ 
& \leq & \modif{\sum_{j>i} a_j + } \sum_{j>i} a_j \int_1^{+\infty} 2^{-u^2} ~du \\ 
&  \lesssim &  \sum_{j>i} a_j \\ 
& = & \sum_{j>i} 4 \cdot 2^{-j} \sqrt{\log(K 2^{j-i} M_{2^{-j}})} \\ 
& = & \sum_{j>i} 4 \cdot 2^{-j} \sqrt{\log(K 2^{j-i} N_{2^{-j}} N_{2^{-(j-1)}})} \\ 
& \lesssim & \sum_{j>i}   2^{-j} \sqrt{j-i} + \sum_{j>i}  2^{-j} \sqrt{ \log \left( K N_{2^{-j}} \right)},
\end{eqnarray*} where the last step used the monotonicity of $N_{\epsilon}$ \modif{so that $N_{2^{-j}} \geq N_{2^{-(j-1)}}$, and $\sqrt{a+b}\leq \sqrt{a}+\sqrt{b}$}. We now observe that because $K \geq 1$ and all the $N_{2^{-j}} \geq 1$ with $N_{2^{-(i+1)}} > 1$ by construction, the first term in the sum above is upper bounded by a constant multiple of the second so that 
\[ E \left[ \sup_{w \in \W} |Z_{w} - Z_{w_0}| \right] \lesssim \sum_{j>i} 2^{-j} \sqrt{\log (K N_{2^{-j}})}\] 
Putting this together with Eq. (\ref{eq:separate_w0}) completest the proof of Eq. (\ref{eq:dudley_discrete_form}). Eq. (\ref{eq:dudley_integral_form}) immediately follows by observing that it is an upper bound on Eq. (\ref{eq:dudley_discrete_form}) due to the fact that $N_{\epsilon}$ is nonincreasing in $\epsilon$.  
\end{proof}

\subsection{Application to Gradient Concentration}

Having proven Theorem \ref{theorem:dudley}, we next explain how we will apply it to prove Theorem \ref{thm:smooth}. 
Recall that our goal is to obtain a concentration result for the quantity
\[ \Delta(S,S')  = \sup_{w \in \W}|| \nabla_w  \ell (S, w) - \nabla_w \ell (S', \omega)||.  \]
Now we will find it convenient to define 
\[ \Delta_{w}(S,S') = || \nabla_w  \ell (S, w) - \nabla_w \ell (S', \omega)||,\] so that we can write compactly 
\[ \Delta(S,S') = E \left[ \sup_{w \in \W} \Delta_{w}(S,S') \right],\] 
\modif{and apply Theorem \ref{theorem:dudley} with $\Delta_{w}(S,S')$ playing the roles of the random variables $Z_{w}$.}

\modif{Our first step is to bound the expected value of $Z_w$ for a single value $w$, which appears as the first term in the bounds \eqref{eq:dudley_discrete_form} and \eqref{eq:dudley_integral_form}. }
It will convenient for us to define  
\[ X_i(w) = \nabla f(w;z_i) - \nabla f(w; z_i'),\] where  recall $S = \{z_1, \ldots, z_n\}$ while $S'=\{z_1', \ldots, z_n'\}$; that way we have in turn 
\begin{equation} \label{eq:deltax} \Delta_{w}(S,S') = \frac{1}{n} \left| \left| \sum_{i=1}^n X_i \right| \right|. 
\end{equation} 

\begin{lemma} Under Assumption \ref{assumption:function}, we have that for any $w \in \W$,
\[ E \left[ \Delta_{w}(S,S') \right] \leq \frac{\aoo{5(\sigma^* + LR)} }{\sqrt{n}}.\] \label{lemma:singlewexp}
\end{lemma}

\begin{proof}  By construction,  the random variables $X_i(w), i = 1, \ldots, n$ are i.i.d. and have zero expectation. Thus, using \eqref{eq:deltax}, we have \begin{small} 
\begin{eqnarray*} 
E \left[ \Delta_{w}(S,S') \right] & \leq & \sqrt{ E \left[ \Delta_{w}(S,S')^2 \right]} \\ 
& = & \frac{1}{n} \sqrt{ E \left[ \left| \left| \sum_{i=1}^n X_i(w) \right| \right|^2 \right]} \\ 
&=& \modif{\frac{1}{n} \sqrt{ n E  \left[  \left| \left|  X_1(w) \right| \right|^2 \right]   }}\\
& = & \frac{1}{\sqrt{n}} \sqrt{ E ||\nabla f(w;z_1) - \nabla f(w; z_1')||^2  } \\ 
& = & \aoo{\frac{1}{\sqrt{n}} \sqrt{ E ||\nabla f(w^*;z_1) + \nabla f(w; z_1) - \nabla f(w^*; z_1) - [\nabla f(w^*; z_1') + \nabla f(w; z_1') - \nabla f(w^*; z_1')]||^2  }}  \\ 
& \leq & \aoo{\frac{1}{\sqrt{n}} \sqrt{ 3 E || \nabla f(w^*; z_1) - \nabla f(w^*; z_1')||^2+ 6 L^2 ||w-w^*||^2}} \\ 
& \leq & 
\aoo{\frac{5(\sigma_* + L R)}{\sqrt{n}}}
\end{eqnarray*}  
\end{small}
\end{proof}

\modif{Our second step towards applying Theorem \ref{theorem:dudley} consists in showing how the random variables $Z_w = \Delta_{w}(S,S')$ satisfy Assumption \ref{ass:increment_concentration}.}


\begin{lemma}\label{lem:increment_concentration}
Setting
\[ Z_w = \Delta_{w}(S,S'),\] we have that  Assumption \ref{ass:increment_concentration} holds with $K=d$ and the metric
    \begin{equation}\label{eq:metric_d}
        d(w_1,w_2)=\frac{cL}{\sqrt{n}}\norm{w_1-w_2},
    \end{equation}
for some absolute constant $c$.
\end{lemma}

\begin{proof}
Using the inequality
\[  \left| ||a|| - ||b|| \right|  \leq ||a-b||, \] we first argue that 
\begin{equation}\label{eq:rewriting_ineq_proba}
P \left(  \frac{1}{n} \left| \left|\sum_{i=1}^n X_i(w_1) \right| \right| - \frac{1}{n} \left| \left|\sum_{i=1}^n X_i(w_2) \right| \right| \geq x \right) \leq P \left( \frac{1}{n} \left| \left|\sum_{i=1}^n X_i(w_1) - X_i(w_2) \right| \right| \geq x \right) \end{equation}
Let us now define 
\[ Y_i = X_i(w_1) - X_i(w_2).\] We have that 
\begin{align*}
    ||Y_i|| &= \norm{ \prt{\nabla f(w_1;z_i) - \nabla f(w_1; z_i')} -\prt{\nabla f(w_2;z_i) - \nabla f(w_2; z_i')}}  \\
     &= \norm{ \prt{\nabla f(w_1;z_i) - \nabla f(w_2;z_i) }  - \prt{\nabla f(w_1;z_i') - \nabla f(w_2;z_i') } }\\
     &\leq \norm{ \nabla f(w_1;z_i) - \nabla f(w_2;z_i) }  +  \norm{\nabla f(w_1;z_i') - \nabla f(w_2;z_i') } \\ 
     & \leq U:=2L||w_1-w_2||,
\end{align*} where the final inequality follows from Assumption \ref{assumption:function}. This implies that, conditional on $Y_1, \ldots, Y_{i-1} $, we have that $Y_i$ is \modif{sub-Gaussian} with parameter $U$. 

Since the elements $z_i$ and $z_i'$ are generated i.i.d,  we further have that  
\[ E \left[ Y_i | Y_{i-1}, \ldots, Y_1 \right] = 0. \]
\modif{These observation imply that the assumptions of Corollary 7 from \cite{jin2019short}, which gives a concentration result for sums of \modif{sub-Gaussian} random vectors, are satisfied. Its application implies the existence of}
an absolute constant $c'$ such that with probability at least $1-\delta$, 
\[ 
\left| \left| 
\sum_{i=1}^n Y_i 
\right| \right| \leq c' U \sqrt{n \log \frac{2d}{\delta}},
\]
and, thus, thanks to the definition of $Y_i$ and Eq. \eqref{eq:rewriting_ineq_proba},
$$
P \left(  \frac{1}{n} \left| \left|\sum_{i=1}^n X_i(w_1) \right| \right| - \frac{1}{n} \left| \left|\sum_{i=1}^n X_i(w_2) \right| \right| \geq   c' U \sqrt{ \frac{1}{n} \log \frac{2d}{\delta}} \right) \leq \delta
$$
If we set 
\[ u = c' U \sqrt{\frac{1}{n} \log \frac{2d}{\delta}},\] then 
we have
$$
e^{u^2n/(c'^2 U^2 )} =  2d/\delta
$$
or 
\begin{equation}\label{eq:exp_bound_delta_d}
    \delta = 2d e^{-u^2n/(c'^2 U^2)}
\end{equation} 
so, \modif{remembering $U = 2L||w_1-w_2||$,}
$$
P \left(  \frac{1}{n} \left| \left|\sum_{i=1}^n Z_i(w_1) \right| \right| - \frac{1}{n} \left| \left|\sum_{i=1}^n Z_i(w_2) \right| \right| \geq  u \right) \leq 2d e^{-u^2n/(c^2 U^2)}   = 2d e^{-\frac{u^2}{d(w_1,w_2)^2}} 
$$
with $d(w_1,w_2)$ defined as in \eqref{eq:metric_d} and $c=2c'$. Using Eq. (\ref{eq:deltax}), this implies the statement of the lemma. 
\end{proof}

\modif{Our third step is to obtain a bound on the covering numbers appearing in the integral in \eqref{eq:dudley_integral_form} or the sum \eqref{eq:dudley_discrete_form}.} The following is a standard estimate. 

\begin{lemma}\label{lem:covering_number}
\modif{For the metric $d$ in \eqref{eq:metric_d}, and a set $\W \subseteq B(0,R)$ for the Euclidean unit ball, we have
$$
N( \epsilon, B_{||\cdot||}(0,R), d)\leq \left(1 + \frac{ \epsilon_0}{\epsilon} \right)^d
$$
with 
\[ \epsilon_0 = \frac{2c LR}{\sqrt{n}}.\]  Moreover, the radius $\widehat R$ of $\W$ in the metric $d$ is bounded as $\widehat R\leq \frac{1}{2}\epsilon_0$.}
\end{lemma}

\begin{proof}
We use the fact that for any norm $||\cdot||'$ over $\R^d$, \modif{the covering number of its corresponding unit ball is bounded by}
$$
N(\epsilon, B_{||\cdot||'}(0,1),||.||') \leq \prt{1+\frac{2}{\epsilon}}^d.
$$ This is follows from standard volume considerations, see \cite{bartlett2013theoretical}. It implies $$
 N(\epsilon, B_{||\cdot||'}(0,R),||.||') \leq \prt{1+\frac{2R}{\epsilon}}^d
$$
However, note that $R$ is a bound on the radius of $\W$ in the ordinary Eucliden norm, not in the $d(\cdot, \cdot)$ so that this bound cannot be applied directly. Rather, we observe that since $d(w_1,w_2) = \frac{cL}{\sqrt{n}}||w_1-w_2||$, we have 
$$N(\epsilon, B_{||\cdot||}(0,R), d) \leq N \left(\frac{\sqrt{n}} {cL}\epsilon, B_{||\cdot||}(0,R),||.|| \right),
$$ 
Therefore, we have
\begin{equation}\label{eq:topological_entropy_our_metric}
N \left(\epsilon, B_{||\cdot||} (0,R),d \right)  \leq \prt{1+\frac{2R}{\frac{\sqrt{n}} {cL}\epsilon}}^d 
= \prt{1+\frac{2cRL}{ \epsilon\sqrt{n}}}^d 
\end{equation} 
\modif{from which the first part of the result follows. The second part is an immediate consequence of $d(w_1,w_2) = \frac{cL}{\sqrt{n}}||w_1-w_2||$ and $\W \subseteq B(0,R)$.}
\end{proof}

\modif{All the pieces are now ready to prove Lemma \ref{lemma:concentration} using Theorem \ref{theorem:dudley}, which concludes the proof of our main result, Theorem \ref{thm:smooth}.}

\medskip
\begin{proof}[Proof of Lemma \ref{lemma:concentration}]
In the notation we introduced earlier, 
\begin{eqnarray*} E \left[ \Delta(S,S') \right] & = & E \left[ \sup_{w \in \W} \Delta_{w}(S,S') \right] \\ 
& = & E \left[ \sup_{w \in \W} Z_w \right]
\end{eqnarray*} 
\modif{Lemma \ref{lem:increment_concentration} shows that Assumption \ref{ass:increment_concentration} is satisfied with the constant $K$ equal the dimension $d$, and the metric $d(w_1,w_2)$ defined as $d(w_1,w_2)=(cL/\sqrt{n}) ||w_1 - w_2||$.}
\modif{Theorem \ref{theorem:dudley} therefore implies}
\begin{equation} \label{eq:firstwbound} E \left[ \sup_{w \in \W} Z_w \right] \lessim E Z_{w_0} + \int_0^{\widehat{R}} \sqrt{\log \left( d N_{\epsilon} \right)} ~d \epsilon. 
\end{equation} 
\modif{Using Lemma \ref{lem:covering_number} to bound $N_\epsilon$ and $\widehat{R}$ then yields}
\begin{equation} \label{eq:secondwbound} E \left[ \Delta(S,S') \right] \lessim E Z_{w_0} + \int_0^{\epsilon_0/2} \sqrt{\log \left( d  \left( 1 + \frac{\epsilon_0}{\epsilon} \right)^d \right)} ~d \epsilon. 
\end{equation}

We next bound the second term on the right-hand side of this equation. Using $\sqrt{a+b}\leq \sqrt{a}+\sqrt{b}$, we have 
\begin{align*}
    \int_0^{\epsilon_0/2} \sqrt{\log \left( d  \left( 1 + \frac{\epsilon_0}{\epsilon} \right)^d \right)} ~d \epsilon & \leq \int_{0}^{\epsilon_0/2}\prt{\sqrt{\log d}+\sqrt{d \log \prt{1+\frac{\epsilon_0}{\epsilon}}}}d\epsilon\\
    & = \frac{\epsilon_0}{2} \sqrt{\log d} + \sqrt{d} \int_{0}^{\epsilon_0/2} \sqrt{ \log \prt{1+\frac{\epsilon_0}{\epsilon}}}d\epsilon\\
    & = \frac{\epsilon_0}{2} \sqrt{\log d}  + \epsilon_0 \sqrt{d}\int_{0}^{1/2}  \sqrt{ \log \prt{1+\frac{1}{x}}}dx \\
   & \lesssim \epsilon_0\sqrt{d}  \\ &   \lesssim  LR \sqrt{\frac{d}{n}} \label{eq:bound_integral},
\end{align*} 
\modif{where we have used the well-definedness of the integral in the third line.} Plugging this into Eq. (\ref{eq:secondwbound}) and using Lemma \ref{lemma:singlewexp} to bound the first term of that equation \modif{completes the proof as}
\[ E \left[ \Delta(S,S') \right] \lesssim \frac{\aoo{\sigma^* + LR}}{\sqrt{n}} + \frac{LR\sqrt{d}}{\sqrt{n}}.\] 
\end{proof}

\section{Inevitability of Dimension Dependence\label{sec:inevitability}}


We now show that some kind of dimension dependence in generalization is inevitable under our assumptions. In particular, we will show that holding $n$ (size of dataset) fixed and sending $d$ (dimension of \modif{the parameter space}) to infinity while keeping the various Lipschitz constants bounded results in a generalization error that is a constant independent of $n$. This rules out the possibility of getting a \modif{uniform} bound that scales as $1/\sqrt{n}$ without any dependence on $d$. 


We will need to tweak our notation slightly \modif{in order to introduce a dimension dependency}. 
\modif{We will work with families of loss functions} $f_{d}(w;z)$ and distributions ${\cal D}_d$, and it will be convenient for our argument to take distributions over $\mathbb{R}^d$. We consider the generalization error at optimality, defined as: 
\[ {\rm G} (f_d, {\cal D}_d) = E_{z \sim {\cal D}_d,\modif{S_n\sim {\cal D}_d }} [f(w_n^*(S_n);z) - \ell (w_n^*(S_n),S_n)],\] 
where 
\[ w_{n}^*(S_n) \in \arg \min_{w \in {\cal W}} \ell(w, S_n), \] where every element of $S_n$ is drawn i.i.d. from ${\cal D}_d$. When the optimizer $w_n^*(S_n)$ is not unique, we modify the definition above to set $G(f_d, {\cal D}_d)$ is defined as the worst generalization error over all possible $w_n^*(S_n)$. Note that the  difference from our earlier definition of generalization is that we evaluate it at $w_n^*(S_n)$, the optimal solution for the sample $S_n$, rather than the iterate $w_t(S_n)$. 

We then have the following proposition. 

\begin{proposition} \label{prop:lower} There exist a family of loss functions $f_d(w;z)$ and  distributions ${\cal D}_d$ such that for any fixed $n$, we have that $w_n^*(S_n)$ is \modif{unique a.a.s. with $d$}, and  
\[ \lim_{d \rightarrow +\infty} {\rm G} (f_d, {\cal D}_d) \geq \frac{1}{8}. \]
Moreover $\sup_{w \in {\cal W}} ||\nabla f_d(w;z)||  \leq 2$ and $f_d(w;z)$ is $1$-smooth for every $d$ and $z$. 
\end{proposition} 

Although this proposition provides an impossibility result for fully dimension-independent bounds at an optimal solution, it can also rule out the possibility of providing \modif{uniform} bounds on SGD iterates such as those we obtained in this paper, \modif{or any algorithm that converges to the minimizer.}  Indeed, given any bound on the generalization error of $w_t$, we can consider what happens to that bound when we set $t \rightarrow \infty$. Given that the minimizer described in the proposition is unique with \modif{high probability}, we will recover the generalization bound on the minimizer. In particular, this rules out the possibility of getting uniform generalization bounds on $w_t$ that scale as $1/\sqrt{n} + 1/\sqrt{T}$ and do not have any dependence on dimension. 

{\aoo The proof of this proposition is quite similar to arguments made in \cite{shalev2009stochastic} and \cite{amir2021sgd}.} 

\bigskip

\begin{proof}[Proof of Proposition \ref{prop:lower}] For convenience, we will write $w_n^*$ instead of $w_n^*(S_n)$, and likewise we will drop the subscript $d$ on $f_d, {\cal D}_d$. We first give a proof of this theorem without the stipulation that $w_n^*$ is unique, and then proceed to sketch how to modify it to ensure uniqueness.
\begin{figure}
  \centering
  \begin{subfigure}{0.45\textwidth}  
    \includegraphics[width=\linewidth]{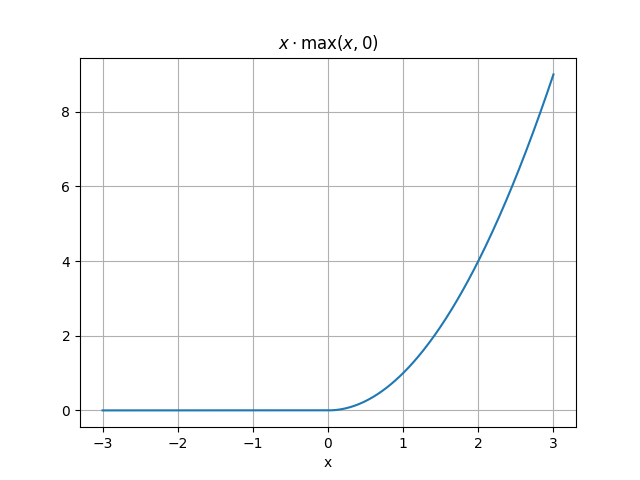}
    \caption{The function $x \max(x,0)$.}
    \label{fig:plot1}
  \end{subfigure}
  \hfill  
  \begin{subfigure}{0.45\textwidth}  
    \includegraphics[width=\linewidth]{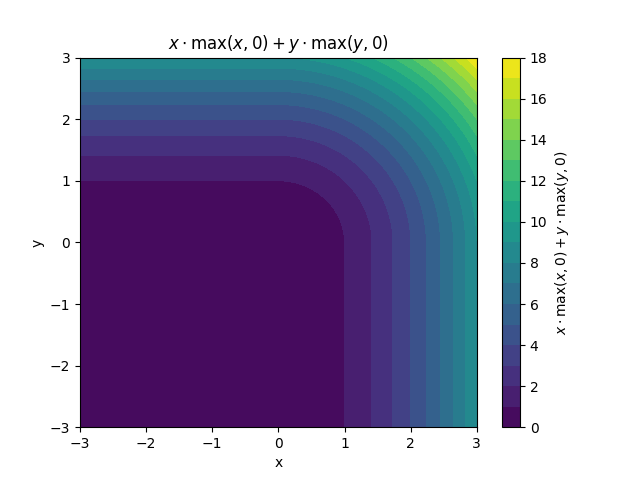}
    \caption{Level sets of $x \max(x,0) + y \max(y,0)$.}
    \label{fig:plot2}
  \end{subfigure}
  \caption{Loss function used in our counterexample.}
  \label{fig:both_plots}
\end{figure}
We define (a family of) loss functions:
\begin{equation}\label{eq:second_loss}
f(w;z) = \sum_{k=1}^d   \frac{1}{2} (z^k w^k)^2_+,
\end{equation} where $w^k,z^k$ refer to the $k$'th coordinate of $w,z$ respectively, and 
$$(x)_+^2  = x\max(x,0).$$ In words, as long as the $k$'th coordinate of $w^k$ and $z^k$ have opposite signs, the loss incurred in the $k$'th term is zero, and otherwise proportional to their product. A plot of this loss function can be found in the figures above. The distribution ${\cal D}_d$ is such that every entry of $z^k$ is chosen to be $\pm 1$ independently. Finally, we define ${\cal W}$ to be the ball of radius $1$ around the origin. 

Let us observe that $\nabla f(w;z)$ is a vector whose $k$'th entry is either $w^k$ or zero. Indeed, if $w^k$ and $z^k$ have the opposite sign, it is zero; and otherwise, it is $(z^k)^2 \modif{w^k = w^k}$, using the fact that $z^k = \pm 1$.  
Consequently $\sup_{w \in {\cal W}} ||\nabla f(w;z)|| = 1$. From the same observation, we see that $\nabla f(w;z)$ is $1$-Lipschitz.

Let $I_{+}$ be the event that there exists a coordinate $k_0$ such that every vector in $S_n = \{z_1, \ldots, z_n\}$ has the $k$'th coordinate equal to $+1$. Likewise, let $I_{-1}$ be the event that there exists a coordinate such that every vector in $S_n$ has that coordinate equal to $-1$. Clearly, 
\[ \lim_{d \rightarrow \infty} P(I_{+} \cap I_{-}) = 1.\]

Now let us define $I = I_{+} \cap I_{-}$ and let us condition on the event $I$. In that case, let $m$ be the number of coordinates where every entry of the vectors $z_i$ is the same (either $+1$ or $-1$). Without loss of generality, suppose the first $\ell$ coordinates have every entry of $z_i$ equal to $+1$ while the next $m-l$ coordinates have every entry of $z_i = -1$. In that case, we observe that 
\begin{equation} \label{eq:firstw*} w_n^* = \frac{1}{\sqrt{m}}(\underbrace{-1, \ldots, -1}_{l}, \underbrace{1, \ldots, 1}_{m-l}, 0, \ldots, 0)\end{equation} 
is a minimizer of $\ell(w;S_n)$ over ${\cal W}$ achieving a loss of zero \modif{since $w^k z^k_i \leq 0$ for every $k,i$ (though clearly not the only minimizer since the zero vector also achieves a cost of zero).} 
We then have that on a new point $z$,  
\begin{align*}
f(w_n^*;z) &= \frac{1}{2}\sum_{k=1}^d (z^k w^k)^2_+\\
&\geq \frac{1}{2}\sum_{k=1}^l \left( -z^k \frac{1}{\sqrt{m}}  \right)^2_+ + \frac{1}{2}\sum_{k=l}^m \left( z^k \frac{1}{\sqrt{m}}  \right)^2_+ \\
&= \frac{1}{2m} \sum_{k=1}^l (-z^k)_+^2 +  \frac{1}{2m} \sum_{k=l}^m (z^k)_+^2,
\end{align*} so that, taking expectations over the distribution of the new point, 
\begin{equation} \label{eq:genlowerbound}  E_{z \in {\cal D}_d} f(w_n^*;z) \geq \frac{1}{2m} \left( \sum_{k=1}^l \frac{1}{2} + \sum_{k=l}^m \frac{1}{2} \right) = \frac{1}{4}. 
\end{equation} 
\modif{As this is true for any realizations $S_n$ satisfying $I$, whose probability converges to 1 when $d$ grows, the generalization error remains larger than 1/8 for large $d$.}
The proof is not over, however, because $w_n^*$ is not the unique minimizer. 
\modif{Hence, we now discuss how to modify the argument to ensure the uniqueness with high probability.} 


We achieve this by adding a linear term to the loss function
\begin{equation}\label{eq:second_loss2}
f_{\epsilon}(w;z) = \epsilon z^Tw+\frac{1}{2}\sum_{k=1}^d (z^k w^k)^2_+ = \sum_{k=1}^d  \left(\epsilon z^kw^k+   \frac{1}{2} (z^k w^k)^2_+\right),
\end{equation} for some $\epsilon>0 $. Note that $\nabla f(w;z)$ is still $1$-Lipschitz and $\lim_{\epsilon \rightarrow 0} \sup_{w \in {\cal W}} ||\nabla f(w;z)|| = 1$.

We fix again realization $S_n$ satisfying event $I$. We use $w_{n, {\rm prev}}$ to refer to $w_n^*$ from Eq. (\ref{eq:firstw*}), and $w_{n, \epsilon}^*$ to refer to the minimizer of the modified empirical loss $\ell_\epsilon$ based on $f_{\epsilon}(w;z)$ as a function of $\epsilon$, which we will show below is unique for any $\epsilon >0$. We will further show that when $\epsilon$ approaches zero, 
\begin{equation} \label{eq:limit} 
\lim_{\epsilon \rightarrow 0} w_{n, \epsilon}^* = w_{n, {\rm prev}},
\end{equation}
and that $E_{z \in {\cal D}_d} f(w_{n, \epsilon}^*;z)$ approaches $1/4$. 
These convergences are uniform over all realizations $S_n$ satisfying $I$ since the number of these is finite. Together with the boundedness of $\ell_\epsilon$ and $\lim_{d\to \infty} P(I)=1 $, this implies that for any large enough $d$, one can take a sufficiently small $\epsilon_d$ for which the generalization error, obtained by further taking expectation over the realizations $S_n$, is larger than $1/8$. At the same time, the minimizer will be unique a.a.s. since we will show below that it is unique conditional on $I$ and of course  $\lim_{d\to \infty} P(I)=1 $.



We now turn to the proof of Eq. (\ref{eq:limit}) as well as of the uniqueness of $w_n^*$. Since we assume $I$ holds, we suppose that dimensions were sorted as above, i.e. all $z_i^k=1$ for $k=1,\dots,l$ and all $z_i^k=-1$ for $k=l+1,\dots,m$.
Observe first that for any of the first $l$ dimensions, 
\begin{equation}\label{eq:partial_deriv}
    \frac{\partial\ell_\epsilon}{\partial w^k}(w;z) =\frac{1}{n}\sum_{i=1}^n  ( \epsilon + \max(0,w^k)) = \epsilon  + \max(0,w^k)> 0,
\end{equation}
(with opposite result for $k=l+1,\dots,m$).
Hence the gradient is never zero, which directly implies that any minimum $w_{n, \epsilon}^*$ lies on the boundary of $\W$, i.e. $||w_{n, \epsilon}^*||=1$. Moreover the minimum is unique because, if there was another minimum, any convex combinations of these two points would also be a minimum by convexity of $\ell_\epsilon(.;z)$, while being in the interior of the unit ball $\W$, which we have just seen is impossible.
The positivity of the partial derivative for $k=1,\dots,l$ further implies that $(w_{n, \epsilon}^*)^k <0$ (where the superscript does not refer to a power, but rather to the $k$'th entry of the vector). All these partial derivatives are equal to $\epsilon$, implying the equality of all $(w_{n, \epsilon}^*)^k$ to some constant $-c_\epsilon$ for $k=1,\dots,l$, since the gradient must be proportional to $w_{n, \epsilon}^*$ by KKT conditions. A parallel argument shows that $(w_{n, \epsilon}^*)^k= c_\epsilon$ for $k=l+1,\dots,m$, i.e. when all $z_i^k$ are -1.

Let us now move to the other dimensions, for which some $z_i^k$ in the dataset are +1 and some others -1, and argue that $\lim_{\epsilon \rightarrow 0} (w_{n, \epsilon}^*)^k = 0$, still for the same fixed realization. Indeed, increasing the $k$'th coordinate above $O(\epsilon)$ or below $-O(\epsilon)$ for such dimension would only increase $f_{\epsilon}$ while making the norm of the vector larger.

These observations allow showing the convergence in Eq. (\ref{eq:limit}): indeed, all coordinates outside of $\{1,\ldots,m\}$ approach zero, while coordinates in $1, \ldots, l$ are always equal to some $-c_\epsilon$ and those in $\{l+1, \ldots, m\}$ to $c_\epsilon$. The fact that the minimizer lies on the boundary of the unit ball $\W$ implies then $c_\epsilon \to 1/\sqrt{m}$.

\end{proof}

\section{Conclusion} We have proved a bound on the generalization error associated with convex gradient descent without assuming strong convexity. We have shown that as $T \rightarrow \infty$, this generalization error saturates at something that scales as $\sim 1/\sqrt{n}$, treating the various Lipschitz parameters, bounds, and dimension as constants. In particular, the implication of this is that early stopping is not required; in contrast, all previous work under these assumptions used early stopping.  The core technical difference between this work and the previous literature is that we do not use arguments based on algorithmic stability when a {\em single} point is removed from the dataset. 

Our bound has an explicit factor of $\sqrt{d}$ in it, but as we have seen, some kind of scaling with dimension is unavoidable under our assumption. An open question is whether one can prove a lower bound on generalization error of SGD that elucidates on scaling with dimension. 

\section{Acknowledgements} The work of J.H. is supported by the Incentive Grant for Scientific Research (MIS) ``Learning from Pairwise Data'' and by the KORNET project from F.R.S.-FNRS and by the ``RevealFlight'' ARC from the fédération Wallonie-Bruxelles. The work A.O. was supported by NSF awards 2245059 and 1914792.


\vskip 0.2in
\bibliography{sample}

\end{document}